\documentclass{article}

\usepackage[draft]{ifdraft}

\PassOptionsToPackage{numbers, sort, compress}{natbib} 

\usepackage[final]{nips_2017}

\usepackage[utf8]{inputenc} 
\usepackage[T1]{fontenc}    
\usepackage{hyperref}       
\usepackage{url}            
\usepackage{booktabs}       
\usepackage{amsfonts}       
\usepackage{nicefrac}       
\usepackage{microtype}      
\usepackage[]{amsmath}
\usepackage{amsthm}
\usepackage{graphicx}
\usepackage{algorithm}
\usepackage[noend]{algpseudocode}
\usepackage{multirow}
\usepackage{subcaption}

\usepackage{booktabs}

\newtheorem{proposition}{Proposition}

\usepackage[dvipsnames]{xcolor}
\usepackage{tikz}
\usetikzlibrary{math}
\usetikzlibrary{arrows.meta}
\usepackage{tikzscale}

\newcommand{\KL}[2]{\mbox{KL}\left[ #1 \,\|\, #2 \right]}
\newcommand{\Exp}[1]{{E}\left[ #1 \right]}
\newcommand{\Expb}[2]{{E}_{#1}\left[ #2 \right]}

\newcommand{\lr}{\mbox{\tiny\textsc{lr}}}
\newcommand{\Ogan}{\mathcal{O}_{\mbox{\tiny\sc gan}}}

\title{\acronym: Reducing Mode Collapse in GANs using Implicit Variational Learning}

\author{
  Akash Srivastava\\
  School of Informatics \\
  University of Edinburgh \\
  \texttt{akash.srivastava@ed.ac.uk} \\
  \And
  Lazar Valkov \\
  School of Informatics \\
  University of Edinburgh \\
  \texttt{L.Valkov@sms.ed.ac.uk} \\
  \AND
  Chris Russell \\
  The Alan Turing Institute \\
  London \\
  \texttt{crussell@turing.ac.uk} \\
  \And
  Michael U. Gutmann \\
  School of Informatics \\
  University of Edinburgh \\
  \texttt{Michael.Gutmann@ed.ac.uk} \\
  \And
  Charles Sutton \\
  School of Informatics \& The Alan Turing Institute \\
  University of Edinburgh \\
  \texttt{csutton@inf.ed.ac.uk} \\
}

\usepackage{sutton_math_symbols}
\usepackage{xspace}
\newcommand{\acronym}{\textsc{Veegan}\xspace}
\newcommand{\reconstructor}{reconstructor network\xspace}

\newcommand{\comment}[1]{} 

\newcommand{\AS}[1]{\comment{\textcolor{BrickRed}{\textit{\lbrack AS: #1 \rbrack}}}}

\newcommand{\MG}[1]{\comment{\textcolor{Orange}{\textit{\lbrack MG: #1 \rbrack}}}}
 
\begin{document}

\maketitle

\begin{abstract}
Deep generative models provide  powerful 
tools for distributions over complicated
manifolds, such as those of natural images.
But many of these methods, including generative adversarial networks (GANs), can be  difficult to train, in part because they are prone to mode collapse, which means that they characterize only a few modes of the true distribution.
To address this,
we introduce \acronym,
which features a reconstructor network, reversing the action of the generator by mapping from
data to noise.
Our training objective retains
the original asymptotic consistency guarantee of  GANs, 
and can
be interpreted as
a novel autoencoder loss over the noise. 
In sharp contrast to a traditional
autoencoder over data points, \acronym does not require specifying
a loss function over the data,
but rather only over the representations, which are standard normal by assumption.
On an extensive set of synthetic  and real world 
image datasets, \acronym indeed
resists mode collapsing
to a far greater extent than other recent GAN variants, and produces more realistic samples.
\end{abstract}


\section{Introduction}
\label{sec:intro}

Deep generative models are a topic of enormous recent interest, providing a powerful class of tools for the unsupervised learning of probability distributions
over difficult manifolds such as natural images \cite{kingma2013auto,rezende2014stochastic,gan}. Deep generative models are usually 
implicit statistical models \citep{diggle:gratton}, also called implicit probability distributions, meaning that
they do not induce a density function that can be tractably computed,
but rather provide a simulation procedure to generate new data points.
Generative adversarial networks (GANs) \cite{gan} are an attractive such method,
which have seen promising recent successes \cite{radford2015unsupervised,CycleGAN2017,sonderby2016amortised}.
GANs train two deep networks in concert: a generator network  that
maps random noise, usually drawn from a multi-variate Gaussian, to data items; and a discriminator network that
estimates the likelihood ratio of the generator network to the data distribution,
and is trained using an adversarial principle.
Despite an enormous amount of recent work, 
GANs are notoriously fickle to train, and it has been observed \cite{salimans16improved,Che2017ModeRG}
that they often suffer from  \emph{mode collapse}, in which the generator
network learns how to generate samples
from a few modes of the data distribution but misses
many other modes, even though samples from
the missing modes occur throughout  the training data.

To address this problem, we introduce \acronym,\footnote{\acronym is a Variational Encoder Enhancement
to Generative Adversarial Networks. \url{https://akashgit.github.io/VEEGAN/}} a variational principle for estimating
implicit probability distributions that avoids mode collapse. While the generator network
maps Gaussian random noise to data items, \acronym introduces an additional \emph{reconstructor} network that
maps the true data distribution to Gaussian random noise.
We train the generator and reconstructor networks jointly by
introducing an implicit variational principle, 
which encourages the reconstructor network
not only to map the data distribution to a Gaussian, but also to approximately reverse
the action of the generator.
Intuitively, if the reconstructor learns
both to map all of the true data to the noise distribution and is an approximate inverse of the generator network, this will encourage
the generator network to map from the noise distribution to the entirety of the true data distribution, thus resolving mode collapse. 


Unlike other  adversarial methods that  train reconstructor networks \citep{donahue2016adversarial,ali2016,tran17dim},
the noise autoencoder dramatically reduces mode collapse.
Unlike recent adversarial methods that also make use of a data autoencoder \cite{larsen2015autoencoding,Che2017ModeRG,makhzani2015adversarial},
\acronym autoencodes noise vectors
rather than data items. This is a significant difference,
because choosing an autoencoder loss for images is problematic,
but for Gaussian noise vectors, an $\ell_2$ loss is entirely 
natural.
Experimentally, on both synthetic and real-world image data sets, we find that \acronym is dramatically less susceptible to mode collapse, and produces higher-quality samples, than other
state-of-the-art methods.

\section{Background}

Implicit probability distributions are  specified 
by a sampling procedure, but do not have a tractable density \cite{diggle:gratton}.
Although a natural choice in many settings, implicit distributions
have historically been seen as difficult to estimate. However, recent progress in formulating  density estimation  as a problem of supervised learning has allowed methods from the classification literature to enable implicit model estimation, both in the general case~\cite{Gutmann2014,1611.10242} and for deep generative adversarial networks (GANs) in particular \cite{gan}. Let $\{ x_i \}_{i=1}^N$ denote the training data, where each $x_i \in \R^D$ is drawn 
from an unknown distribution $p(x)$.
A GAN is a neural network
$G_\gamma$ that maps representation vectors $z \in \R^K$, typically 
 drawn from a standard normal distribution, to data items $x \in \R^D$. 
Because this mapping defines an implicit probability distribution, 
training is accomplished by introducing a second neural network
$D_\omega$, called a discriminator, 
whose goal is to distinguish generator samples from true data samples. The parameters of these  networks are estimated
by solving the minimax problem
\begin{equation*}
\max_\omega \min_\gamma \Ogan (\omega, \gamma) :=
 \Expb{z}{\log \sigma \left(D_\omega(G_\gamma (z))\right)}  + \Expb{x}{\log\left( 1-\sigma \left(D_\omega(x)\right)\right)},
\end{equation*}
where $E_z$ indicates an expectation over the standard normal $z$,
$E_x$ indicates an expectation over the data distribution $p(x)$, and
$\sigma$ denotes the sigmoid function. At the optimum, in the limit of
infinite data and arbitrarily powerful networks, we will have
$D_\omega = \log q_\gamma(x)/p(x)$, where $q_\gamma$ is the density
that is induced by running the network $G_\gamma$ on normally
distributed input, and hence that $q_\gamma = p$ \cite{gan}.

Unfortunately, GANs can be difficult and unstable to train \citep{salimans16improved}. One common pathology
that arises in GAN training is mode collapse, which is when
samples from $q_\gamma(x)$ capture only a few of the modes
of $p(x)$.
An intuition behind why mode collapse occurs
is that the only information that the objective function provides 
about $\gamma$ is mediated
by the discriminator network $D_\omega$. For example, if
$D_\omega$ is a constant, then $\Ogan$ is constant with respect to $\gamma$, and so learning the generator is impossible.
When this situation occurs in a localized region of input space,
for example, when there is a specific type of image that the
generator cannot replicate,
this can cause mode collapse.

\section{Method}
\MG{Serious re-write: please check that you are ok with it and that I didn't say anything stupid or wrong}

The main idea of \acronym is to introduce a second network $F_\theta$
that we call the \emph{\reconstructor} which is learned  both to map the
true data distribution $p(x)$ to a Gaussian and to approximately invert
the generator network.

To understand why this might
prevent mode collapse, consider the example in
\autoref{fig:cartoon_all}. In both columns of the figure, the middle
vertical panel represents the data space, where in this example the
true distribution $p(x)$ is a mixture of two Gaussians.  The bottom
panel depicts the input to the generator, which is drawn from a
standard normal distribution $p_0= \calN(0, I)$,
and the top panel depicts the result of applying
the reconstructor network to the generated
and the true data. The arrows labeled
 $G_\gamma$ show the action of the generator. The purple arrows labelled $F_\theta$ show the action of the reconstructor on the true data,
 whereas the green arrows show the action of the reconstructor on data from the generator. In this example, the
generator has captured only one of the two modes of $p(x)$. The difference between Figure \ref{fig:cartoon1} and \ref{fig:cartoon2} is that the reconstructor networks
are different.

First, let us suppose (Figure \ref{fig:cartoon1}) that we have successfully
trained $F_\theta$ so that it is approximately
the inverse of $G_\gamma$. As we have assumed mode collapse
however, the training data for the reconstructor network $F_\theta$
does not include data items from the ``forgotten" mode of $p(x),$ therefore
 the action of $F_\theta$ on data from that mode is
ill-specified. This means that $F_\theta(X), X \sim p(x)$ is unlikely
to be Gaussian and we can use this mismatch as an indicator of mode
collapse.

Conversely, let us suppose (Figure \ref{fig:cartoon2}) that $F_\theta$ is
successful at mapping the true data distribution 
to a Gaussian.  In that case, if
$G_\gamma$ mode collapses, then $F_\theta$ will not map all $G_\gamma(z)$ back to the original $z$
and the resulting penalty provides us with a strong learning signal
for both $\gamma$ and $\theta$. 

Therefore, the learning principle for \acronym
will be to train $F_\theta$ to achieve
both of these objectives simultaneously. 
Another way of stating this intuition is that
if the same reconstructor network maps both
the true data and the generated data to a Gaussian distribution, then the generated data
is likely to coincide with true data.
To measure whether $F_\theta$ approximately
inverts $G_\gamma$, we use an autoencoder loss.
 More precisely,  we 
minimize a loss function, like $\ell_2$ loss between $z \sim p_0$ 
and $F_\theta(G_\gamma(z)))$.
To quantify whether $F_\theta$ maps
the true data distribution to a Gaussian,
we use the cross entropy $H(Z, F_\theta(X))$ between $Z$ and
$F_\theta(x)$.
This  boils down to learning $\gamma$ and $\theta$ by
minimising the sum of these two objectives, namely
\begin{equation}
  \calO_{\mathrm{entropy}}(\gamma,\theta) = E\left[\| z - F_\theta(G_\gamma(z))\|_2^2\right] + H(Z, F_\theta(X)).\label{eq:intuition}
\end{equation}
While this objective captures the main idea of our paper, it cannot be
easily computed and minimised. We next transform it into a
computable version and derive theoretical guarantees.



\begin{figure}
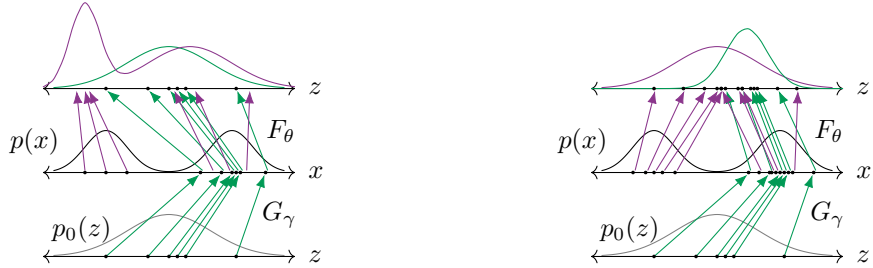

\centering
\begin{subfigure}[t]{0.45\linewidth}
\centering
\includegraphics[width=0.7\textwidth]{collapse_cartoon_2.tikz}
\caption{Suppose $F_\theta$ is
trained to approximately invert $G_\gamma.$
Then applying $F_\theta$ to true data
is likely to produce a non-Gaussian distribution, allowing us to detect mode collapse.\label{fig:cartoon1}}

\end{subfigure}\hspace{1cm}%
\begin{subfigure}[t]{0.45\linewidth}
\centering
\includegraphics[width=0.7\textwidth]{collapse_cartoon_1.tikz}
\caption{When $F_\theta$ is trained to map
the data to a Gaussian distribution, then
treating $F_\theta \circ G_\gamma$ as an autoencoder
provides learning signal to correct $G\gamma$.\label{fig:cartoon2}}
\end{subfigure}
\caption{Illustration of how a reconstructor network $F_\theta$
can help to detect mode collapse in a deep generative
network $G_\gamma$. The data distribution is $p(x)$ and the Gaussian is $p_0(z)$. See text for details.\label{fig:cartoon_all}}
\end{figure}

\subsection{Objective Function}
Let us denote the distribution of the outputs of the reconstructor
network when applied to a fixed data item $x$ by $p_\theta(z|x)$ and
when applied to all $X \sim p(x)$ by $p_\theta(z) = \int p_\theta(z|x)
p(x) \,dx$. The conditional distribution $p_\theta(z|x)$ is Gaussian
with unit variance and, with a slight abuse of notation,
(deterministic) mean function $F_\theta(x)$. The entropy term $H(Z,
F_\theta(X))$ can thus be written as
\begin{align}
H(Z, F_\theta(X)) = -\int p_0(z) \log p_\theta(z)dz = -\int p_0(z) \log \int p(x) p_\theta(z|x) \,dx \,dz.\label{eq:xe}
\end{align}
This cross entropy is minimized with respect to $\theta$ when
$p_\theta(z) = p_0(z)$ \citep{cover:thomas}.  Unfortunately, the
integral on the right-hand side of \eqref{eq:xe} cannot usually be
computed in closed form. We thus introduce a variational distribution
$q_\gamma (x|z)$ and by Jensen's inequality, we have
\begin{equation}
-\log p_\theta(z) = -\log \int p_\theta(z|x) p(x) \frac{q_\gamma (x|z)}{q_\gamma (x|z)} \,dx 
  \leq -\int q_\gamma(x|z) \log \frac{p_\theta(z|x)p(x)}{q_\gamma(x|z)} \,dx,\label{eq:naive}
\end{equation}
which we use to bound the cross-entropy in \eqref{eq:xe}. In
variational inference, strong parametric assumptions are typically
made on $q_\gamma$. Importantly, we here relax that assumption, instead
representing $q_\gamma$ implicitly as a deep generative model,
enabling us to learn very complex distributions. The variational
distribution $q_\gamma(x|z)$ plays exactly the same role as the
generator in a GAN, and for that reason, we will
parameterize $q_\gamma(x|z)$ as the output
of a stochastic neural network $G_\gamma(z)$.

In practice minimizing this bound is difficult if $q_\gamma$ is specified implicitly.
For instance, it is challenging to train a discriminator network 
that accurately estimates the unknown likelihood ratio $\log p(x) / q_\gamma(x|z)$, because $q_\gamma(x|z),$ as a conditional distribution, is much more peaked than the joint distribution $p(x)$, making it  too easy for
a discriminator to tell the two distributions apart. Intuitively, the discriminator in a GAN works well
when it is presented a
\emph{difficult} pair of distributions to distinguish.
To circumvent this problem, 
we write (see supplementary material)
\begin{align}
-\int p_0(z) \log p_\theta(z) 
& \leq \KL{q_\gamma(x|z) p_0(z)}{p_\theta(z|x)p(x)}
 - \Exp{\log p_0(z)}. \label{eq:ideal-kl_}
\end{align}
Here all expectations are taken with respect to the joint distribution $p_0(z) q_\gamma(x|z).$ 

Now, moving to the second term in \eqref{eq:intuition},
we define the reconstruction penalty as an expectation of the cost of autoencoding noise vectors, that is,
$\Exp{d(z, F_\theta(G_\gamma(z)))}.$
The function $d$ denotes
a loss function in representation space $\R^K$, such as $\ell_2$ loss and therefore the term is an autoencoder in representation space. To make this link
explicit, we expand the expectation, assuming
that we choose $d$ to be $\ell_2$ loss. This yields
$\Exp{d(z, F_\theta(x))} =  \int p_0(z) \int q_\gamma(x|z) \|z - F_\theta(x)\|^2 \,dx dz.$ Unlike a standard autoencoder, however, rather than taking a \emph{data item} as input and attempting to reconstruct it, we  autoencode a \emph{representation vector}. 
This makes a substantial difference in the interpretation and performance of the method, as we discuss in \autoref{sec:mode-reg}. For example, notice that we do not
include a regularization weight on
the autoencoder term in \eqref{eq:ideal-kl}, because Proposition 1 below says that this is not needed to recover
the data distribution.

Combining these two ideas, we obtain the final objective function
\begin{align}
\calO(\gamma, \theta) & = \KL{q_\gamma(x|z) p_0(z)}{p_\theta(z|x)p(x)}
 - \Exp{\log p_0(z)}
+ \Exp{d(z, F_\theta(x))}. \label{eq:ideal-kl}
\end{align}
Rather than minimizing the intractable $\calO_{\mathrm{entropy}}(\gamma,\theta)$, our goal in \acronym\ is to  minimize the upper bound $\calO$  with respect to $\gamma$
and $\theta$.
Indeed, if the networks
$F_\theta$ and $G_\gamma$ are sufficiently powerful, then if we succeed
in globally minimizing $\calO,$  we can guarantee that $q_\gamma$ recovers
the true data distribution. This statement is formalized in the following proposition.

\begin{proposition}
Suppose that there exist parameters $\theta^*, \gamma^*$ such that 
$\calO(\gamma^*, \theta^*) = H[p_0],$ where $H$ denotes Shannon entropy.
Then $(\gamma^*, \theta^*)$ minimizes $\calO$, and further 
\begin{equation*}
 p_{\theta^*}(z) := \int p_{\theta^*}(z|x) p(x) \,dx = p_0(z), \quad \mbox{ and } \quad
q_{\gamma^*}(x) := \int q_{\gamma^*}(x|z) p_0(z) \,dz = p(x).
\end{equation*}
\end{proposition}
Because neural networks are universal approximators,
the conditions in the proposition can be achieved when the networks $G$ and $F$ are sufficiently powerful.

\subsection{Learning with Implicit Probability Distributions}

This subsection describes how to approximate $\calO$ when we have implicit representations
for $q_\gamma$ and $p_\theta$ rather than explicit densities. 
In this case, we cannot optimize $\calO$ directly, because 
the KL divergence in \eqref{eq:ideal-kl} depends on a density ratio which is unknown, both
because $q_\gamma$ is implicit and also because $p(x)$ is unknown. Following \cite{ali2016,donahue2016adversarial}, we estimate
this ratio using a discriminator network $D_\omega(x,z)$ 
which we will train to encourage
\begin{equation}
\label{eq:d_opt}
D_{\omega}(z,x) = \log \frac{q_\gamma(x|z)p_0(z)}{p_{\theta}(z | x) p(x)}.
\end{equation}
This will allow us to estimate $\calO$ as
\begin{align}
\label{eq:ohat}
\hat{\calO} (\omega, \gamma, \theta) = \frac{1}{N} \sum_{i=1}^N \mathcal{D}_\omega(z^i, x_g^i) + \frac{1}{N} \sum_{i=1}^N d(z^i, x_g^i),
\end{align}
where $(z^i,x_g^i) \sim p_0(z) q_\gamma(x|z)$.
In this equation, note that $x_g^i$ is a function of $\gamma$;
although we suppress this in the notation, we do take
this dependency into account in the algorithm.
We use an auxiliary objective function to estimate $\omega$.
As mentioned earlier, we omit the entropy term $-\Exp{\log p_0(z)}$ from
$\hat{\calO}$ as it is constant with respect to all parameters.
In principle, any method for density ratio estimation could be used here,
for example, see \cite{sugiyama:book,Gutmann2011b}. In this work, we will use the logistic regression loss, much as in other methods for deep adversarial training, such as GANs \cite{gan}, or for noise contrastive estimation \cite{gutmann12}.
We will train $D_\omega$ to distinguish samples from the joint distribution
$q_{\gamma}(x|z)p_{0}(z)$ from $p_{\theta}(z|x) p(x)$. The objective function for this is
\begin{equation}
\label{eq:o3}
\calO_{\lr}(\omega, \gamma, \theta) = -\Expb{\gamma}{\log\left(\sigma \left(D_\omega(z,x)\right)\right)}-  \Expb{\theta}{\log\left( 1-\sigma \left(D_\omega(z,x)\right)\right)},
\end{equation}
where $E_{\gamma}$ denotes expectation with respect
to the joint distribution $q_\gamma(x|z) p_0(x)$ 
and $E_{\theta}$ with respect to $p_\theta(z|x)p(x)$.
We write $\hat{\calO}_{\lr}$ to indicate the Monte Carlo estimate of $\calO{\lr}$.
Our learning algorithm optimizes this pair of equations with respect to $\gamma, \omega, \theta$ using stochastic gradient descent. In particular, the algorithms aim to
find a simultaneous solution to
$\min_{\omega} \hat{\calO}_{\lr} (\omega, \gamma, \theta)$
and $\min_{\theta,\gamma} \hat{\calO} (\omega, \gamma, \theta)$.
  This training procedure is described in Algorithm~\ref{alg:veegan}.
When this procedure converges, we will have that
$\omega^* = \arg\min_\omega \calO_{\lr}(\omega, \gamma^*, \theta^*)$,
which means that $D_{\omega^*}$ has converged to the likelihood ratio
\eqref{eq:d_opt}. Therefore $(\gamma^*, \theta^*)$ have also
converged to a minimum of $\calO$. 

We also found that pre-training the reconstructor network on samples from $p(x)$ helps in some cases. 

\begin{algorithm}[tb]
\caption{\acronym training\label{alg:veegan}}
\begin{algorithmic}[1]
\While{\text{not converged}}
\For{$i \in \{1\ldots N\}$}
\State Sample $z^i \sim p_0(z)$ 
\State Sample $x_g^i \sim q_\gamma(x|z_i)$ 
\State Sample $x^i \sim p(x)$ 
\State Sample $z_g^i \sim p_\theta(z_g|x_i)$ 
\EndFor
\State $g_\omega \gets - \nabla_\omega \frac{1}{N}\sum_i \log \sigma \left(D_\omega(z^i,x_g^i)\right) + \log\left( 1-\sigma \left(D_\omega(z_g^i,x^i)\right)\right) $
\Comment{Compute $\nabla_\omega \hat{\calO}_{\lr}$ }
\State
\State $g_\theta \gets  \nabla_\theta \frac{1}{N}\sum_i d(z^i, x_g^i)$
\Comment{Compute $\nabla_\theta \hat{\calO}$ }
\State
\State $g_\gamma \gets \nabla_\gamma \frac{1}{N}\sum_i D_\omega(z^i,x_g^i) +  \frac{1}{N}\sum_i d(z^i, x_g^i)$
\Comment{Compute $\nabla_\gamma \hat{\calO}$ }
\State
\State $\omega \gets \omega - \eta g_\omega$; $\theta \gets \theta - \eta g_\theta$; $\gamma \gets \gamma - \eta g_\gamma$
\Comment{Perform SGD updates for $\omega,\theta$ and $\gamma$}
\EndWhile
\end{algorithmic}
\end{algorithm}

\section{Relationships to Other Methods}

An enormous amount of attention has been devoted recently
to improved methods for GAN training,
and we compare ourselves to the most closely related work in detail.

\paragraph{BiGAN/Adversarially Learned Inference}
BiGAN \citep{donahue2016adversarial} and Adversarially Learning Inference (ALI) \citep{ali2016}
are two essentially identical recent adversarial methods for learning both a deep generative network
$G_\gamma$ and a reconstructor network $F_\theta$. 
Likelihood-free variational inference (LFVI) \citep{tran17dim} extends this
idea to a hierarchical Bayesian setting. Like \acronym,
all of these methods also
use
a discriminator $D_\omega(z,x)$ on the joint $(z,x)$ space.  However, the \acronym objective function $\calO(\theta, \gamma)$
provides significant benefits over
the logistic regression loss over $\theta$ and $\gamma$ that is used in ALI/BiGAN, or the KL-divergence used in LFVI.

In all of these methods, just as in vanilla GANs,
the objective function
depends on $\theta$ and $\gamma$ only
via the output $D_\omega(z,x)$ of the discriminator; therefore, if there is a mode of data space in which  $D_\omega$ is insensitive
to changes in $\theta$ and  $\gamma$,
there will be mode collapse.
In \acronym, by contrast, the reconstruction term does not depend on the discriminator,
and so can provide learning signal to $\gamma$ or $\theta$ even when
the discriminator is constant.
We will show in \autoref{sec:experiments} that indeed \acronym is dramatically less prone
to mode collapse than ALI.

\paragraph{InfoGAN} While differently motivated to obtain disentangled representation of the data, InfoGAN also uses a latent-code reconstruction based penalty in its cost function. But unlike \acronym, only a part of the latent code is reconstructed in InfoGAN. Thus, InfoGAN is similar to VEEGAN in that it also includes an autoencoder over the latent codes, but the key difference is that InfoGAN does not also train the reconstructor network on the true data distribution. We suggest that this may be the reason that InfoGAN was observed to require some of the same stabilization tricks as vanilla GANs, which are not required for VEEGAN.

\paragraph{Adversarial Methods for Autoencoders}
\label{sec:mode-reg}
A number of other recent methods have been proposed that combine adversarial methods 
and autoencoders, whether by explicitly regularizing
the GAN loss with an autoencoder loss
\citep{Che2017ModeRG,larsen2015autoencoding},
or by alternating optimization between
the two losses \citep{makhzani2015adversarial}.
In all of these methods, the autoencoder is over
images, i.e., they incorporate a loss function 
of the form $\lambda d(x, G_\gamma (F_\theta (x))),$
where $d$ is a loss function over images, such as
pixel-wise $\ell_2$ loss, and $\lambda$
is a regularization constant.
Similarly, variational autoencoders
\cite{kingma14,rezende2014stochastic}
also autoencode images rather than noise vectors.
Finally, the adversarial variational Bayes (AVB) 
 \citep{avb} is an adaptation of VAEs
to the case where the posterior distribution $p_\theta(z|x)$
is implicit, but the data distribution $q_\gamma(x|z)$, must be explicit, unlike in our work.

Because these methods autoencode data points,
they share a crucial disadvantage.
Choosing a good loss function $d$
over natural images can be problematic.
For example, it has been commonly observed that 
minimizing an $\ell_2$ reconstruction loss
on images can lead to blurry images.
Indeed, if choosing a loss function over images were easy,
we could simply train an autoencoder and dispense
with adversarial learning entirely.
By contrast,  in \acronym we autoencode the noise vectors $z$, and 
\emph{choosing a good loss function for a noise
autoencoder is easy.} The noise vectors $z$
are drawn from a standard normal distribution,
using an $\ell_2$ loss on $z$ is entirely natural
--- and does not, as we will show in \autoref{sec:experiments}, result in blurry images compared
to purely adversarial methods. 

\section{Experiments}\label{sec:experiments}

Quantitative evaluation of GANs is problematic because implicit distributions do not have a tractable likelihood term to quantify generative accuracy. Quantifying mode collapsing is also not straightforward, except  in the case of synthetic data with known modes. 
For this reason, several indirect metrics have recently been
proposed to evaluate GANs specifically for their mode collapsing behavior \citep{metz2016unrolled,Che2017ModeRG}. However, none of these metrics are reliable on their own and therefore 
we need to compare across a number of different methods.
Therefore in this section we evaluate \acronym on several synthetic and real datasets and compare its performance against vanilla GANs \citep{gan}, Unrolled GAN \citep{metz2016unrolled} and ALI \citep{ali2016} on five different metrics. Our results strongly suggest that \acronym does indeed resolve mode collapse in GANs
to a large extent. 
Generally, we found that \acronym performed well with default hyperparameter values, 
so we did not tune these.  Full details are provided in the supplementary material.

\subsection{Synthetic Dataset}

Mode collapse can be accurately measured on synthetic datasets, since the true  distribution and its modes are known. 
In this section we compare all four competing methods on three synthetic datasets of increasing difficulty: 
a mixture of eight 2D Gaussian distributions arranged in a ring, a mixture of twenty-five 
2D Gaussian distributions arranged in a grid \footnote{Experiment follows \cite{ali2016}. Please note that for certain settings of parameters, vanilla GAN can also recover all 25 modes, as was pointed out to us by Paulina Grnarova.} and a mixture of ten 
700 dimensional Gaussian distributions embedded in a 1200 dimensional space.  This mixture arrangement was chosen to mimic the higher dimensional manifolds of natural images.
All of the mixture components were isotropic Gaussians.
For a fair comparison of the different learning methods
for GANs, we use the same 
network architectures for the reconstructors and the generators for all methods,
namely, fully-connected MLPs with two hidden layers. For the discriminator we use a two layer MLP without  dropout or normalization layers. \acronym method works for both deterministic and stochastic generator networks. To allow for the generator to be a stochastic map we add an extra dimension of noise to the generator input that is not reconstructed. 

To quantify the mode collapsing behavior we report two metrics:  We sample points from the generator network, and count
a sample as \emph{high quality},
if it is within three standard deviations of the nearest mode, 
for the 2D dataset, or within 10 standard deviations of the nearest mode, for the 1200D dataset.
Then, we report the \emph{number of modes captured} as the number
of mixture components whose mean is nearest to at least
one high quality sample.
We also report the percentage of high quality samples
as a measure of sample quality. We generate $2500$ samples from each trained model and average the numbers over five runs. 
For the unrolled GAN, we set the number of unrolling steps to five as suggested in the authors' reference implementation.

As shown in \autoref{tab:syn},  \acronym  captures the greatest number of modes on all the synthetic datasets, while consistently generating higher quality samples. 
This is visually apparent in Figure \ref{fig:ring}, which plot the generator distributions for each 
method; the generators learned by \acronym are sharper
and closer to the true distribution. This
figure also shows why it is important to measure sample quality
and mode collapse simultaneously, as either alone can be 
misleading.
For instance, the GAN on the 2D ring has $99.3\%$ sample quality, but this is simply because the GAN collapses all
of its samples onto one mode (Figure \ref{fig:gan_ring}). On the other extreme,
the unrolled GAN on the 2D grid captures almost all the modes in 
the true distribution,
but this is simply because 
that it is generating highly dispersed samples (Figure \ref{fig:unrolled_grid2}) that do not accurately represent the true distribution, hence the low sample quality. All methods had approximately the same running time, except for unrolled GAN,
which is a few orders of magnitude slower due to the unrolling overhead.

\begin{table}
\centering
\caption{Sample quality and degree of mode collapse
on mixtures of Gaussians.  \acronym consistently
captures the highest number of modes and produces better samples.\label{tab:syn}}\vspace{1ex}

\resizebox{\textwidth}{!}
{
\begin{tabular}{@{}lcccccc@{}}
\toprule
\multirow{2}{*}{}     & \multicolumn{2}{c}{\textbf{2D Ring}}                                                                                                          & \multicolumn{2}{c}{\textbf{2D Grid}}                                                                                                           & \multicolumn{2}{c}{\textbf{1200D Synthetic}}                                                                                                   \\  \cmidrule(r{1em}l{1em}){2-3}\cmidrule(r{1em}l{1em}){4-5}\cmidrule(r{1em}l{1em}){6-7}
                      & \textbf{\begin{tabular}[c]{@{}c@{}}Modes \\ (Max 8)\end{tabular}} & \textbf{\begin{tabular}[c]{@{}c@{}}\% High Quality\\ Samples\end{tabular}} & \textbf{\begin{tabular}[c]{@{}c@{}}Modes \\ (Max 25)\end{tabular}} & \textbf{\begin{tabular}[c]{@{}c@{}}\% High Quality\\ Samples\end{tabular}} & \textbf{\begin{tabular}[c]{@{}c@{}}Modes \\ (Max 10)\end{tabular}} & \textbf{\begin{tabular}[c]{@{}c@{}}\% High Quality\\ Samples\end{tabular}} \\ \midrule
\textbf{GAN}          & 1                                                                 & 99.3                                                                       &3.3                                                                    &0.5                                                                            &1.6                                                                    &2.0                                                                            \\ 
\textbf{ALI}          & 2.8                                                               & 0.13                                                                       & 15.8                                                               & 1.6                                                                        & 3                                                                  & 5.4                                                                        \\ 
\textbf{Unrolled GAN} & 7.6                                                               & 35.6                                                                       & 23.6                                                               & 16                                                                         & 0                                                                  & 0.0                                                                         \\ 
\textbf{\acronym}     & \textbf{8}                                                        & \textbf{52.9}                                                              & \textbf{24.6}                                                      & \textbf{40}                                                                & \textbf{5.5}                                                       & \textbf{28.29}                                                             \\ \bottomrule
\end{tabular}
}
\end{table}

\subsection{Stacked MNIST}

Following \citep{metz2016unrolled}, we evaluate our methods on the stacked MNIST dataset, a variant of the MNIST data specifically
designed to increase the number of discrete modes.
The data is synthesized by stacking three randomly sampled MNIST digits along the color channel resulting in a 28x28x3 image.
We now expect $1000$ modes in this data set,
corresponding to the number of possible triples of digits.
 
Again, to focus the evaluation on the difference in the learning algorithms, we use the same generator architecture
for all methods.
In particular, the generator architecture is an off-the-shelf standard implementation\footnote{\url{https://github.com/carpedm20/DCGAN-tensorflow}} of DCGAN \citep{radford2015unsupervised}.

For Unrolled GAN, we used a standard implementation of the DCGAN discriminator network. For  ALI and \acronym, the discriminator 
architecture is described
in the supplementary material. 
For the reconstructor in ALI and \acronym,
we use a simple two-layer MLP for the reconstructor without any regularization layers. 

Finally, for \acronym we  pretrain the reconstructor by taking a few 
stochastic gradient steps with respect to $\theta$ before running Algorithm
\ref{alg:veegan}. For all methods other than \acronym, 
we use the enhanced generator loss function suggested in \citep{gan}, since we were not able to get sufficient learning signals for the generator without it.
\acronym did not require this adjustment for successful training.

As the true locations of the modes in this data 
are unknown, 
the number of modes are estimated using a trained classifier as described originally in \citep{Che2017ModeRG}. We used a total of $26000$ samples for all the models and the results are averaged over five runs. As a measure of quality, 
following \citep{metz2016unrolled} again, we also report the KL divergence between the generator distribution and the data distribution.
As reported in Table \ref{tab:mnist}, \acronym  not only  captures the most  modes, it  consistently matches the data distribution more closely than any other method. Generated samples from each of the models are shown in the supplementary material.

\begin{table}
\centering
\begin{tabular}{@{}ccc|r@{ $\pm$ }l@{}}\toprule
\multicolumn{1}{l}{\multirow{2}{*}{}} & \multicolumn{2}{c}{\textbf{Stacked-MNIST}} & \multicolumn{2}{c}{\textbf{CIFAR-10}}                                  \\   \cmidrule(r{1em}l{1em}){2-3} 
               & {Modes (Max 1000)} & \multicolumn{1}{c}{KL} & \multicolumn{2}{c}{{IvOM}}                                \\ \midrule
\textbf{DCGAN}                          & 99                    & 3.4   & 0.00844 & 0.002       \\ 
\textbf{ALI}                            & 16                        & 5.4 & \textbf{0.0067} & 0.004                                                  \\ 
\textbf{Unrolled GAN}                   & 48.7                       & 4.32  & 0.013 & 0.0009                                                     \\
\textbf{\acronym}                       & \textbf{150}             & \textbf{2.95}      & \textbf{0.0068} & 0.0001                                    \\ \bottomrule
\end{tabular} \vspace{1ex}
\caption{Degree of mode collapse, measured by modes captured and the inference via optimization measure (IvOM), and sample quality (as measured by KL)
on Stacked-MNIST and CIFAR.  \acronym captures the most modes and also achieves the highest quality.}
\label{tab:mnist}
\label{tab:cifar}
\end{table}

\subsection{CIFAR}

Finally, we evaluate the learning methods on the CIFAR-10 dataset, a well-studied and diverse dataset of natural images. 
We use the same discriminator, generator, and reconstructor
architectures as in the previous section. 
However, the previous mode collapsing metric 
is inappropriate here, owing to CIFAR's greater diversity.
Even within one of the 10 classes of CIFAR,
the intra-group diversity is very high compared to any of the  10 classes of MNIST. Therefore, for CIFAR it is inappropriate to assume, as the metrics of the previous subsection do, that each
labelled class corresponds to a single mode of the data distribution.

Instead, we use a metric introduced by \citep{metz2016unrolled} which we will call the inference via optimization metric (IvOM). The idea behind this metric is to compare real images from the test
set to the nearest generated image; if the generator suffers
from mode collapse, then there will be some images for
which this distance is large.
To quantify this,
we sample a real image $x$ from the test set, and find
the closest image that the GAN is capable of generating, i.e.  optimizing the $\ell_2$ loss between $x$ and generated image $G_\gamma(z)$ with respect to $z$. 
 If a method consistently attains low MSE, then it can be assumed to be capturing more modes than the ones which attain a higher MSE. 
As before, this metric can still be fooled by highly dispersed generator distributions, and also the $\ell_2$ metric may favour
generators that produce blurry images. Therefore we will also
evaluate sample quality visually.
All numerical results have been averaged over five runs.
Finally, to evaluate whether the noise autoencoder
in \acronym is indeed superior to a more traditional data autoencoder, we compare to a variant, which we call \acronym+DAE, that uses
a data autoencoder instead, by simply replacing $d(z, F_\theta(x))$
in $\calO$ with a data loss $\|x - G_\gamma(F_\theta(x)))\|^2_2$.


As shown in Table \ref{tab:cifar}, ALI and \acronym achieve the best IvOM.  
Qualitatively, however, generated samples from \acronym
seem better than other methods. In particular, the samples from \acronym+DAE are
meaningless. Generated samples from \acronym are shown in \autoref{fig:cifar_generated}; samples from other methods
are shown in the supplementary material. As another illustration of this, Figure \ref{fig:cifar} illustrates the IvOM metric, by showing
the nearest neighbors to real images that each of the GANs
were able to generate; in general, the nearest neighbors
will be more semantically meaningful than randomly generated
images. We omit \acronym+DAE from this table because it did not produce plausible images.
Across the methods, we see in \autoref{fig:cifar}  that
\acronym captures small details, such as the face of the poodle, that other methods miss.


\begin{figure}[!htb]
 \caption{Density plots of the true data and generator distributions from  different GAN methods
 trained on mixtures of Gaussians arranged in a ring (top) or a grid (bottom).\label{fig:ring}\label{fig:grid}}
\begin{subfigure}{0.18\textwidth}
  \includegraphics[width=\linewidth]{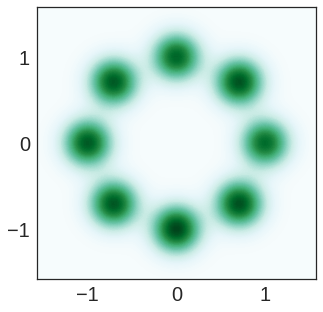}
  \caption{True Data\label{fig:true_ring}}
\end{subfigure} 
\begin{subfigure}{0.18\textwidth}
  \includegraphics[width=\linewidth]{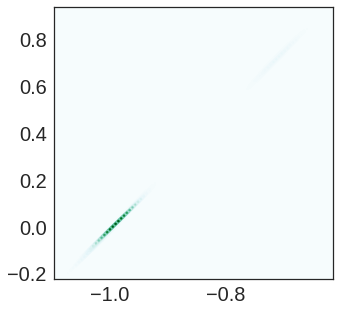}
  \caption{GAN\label{fig:gan_ring}}
\end{subfigure}
\begin{subfigure}{0.18\textwidth}
  \includegraphics[width=\linewidth]{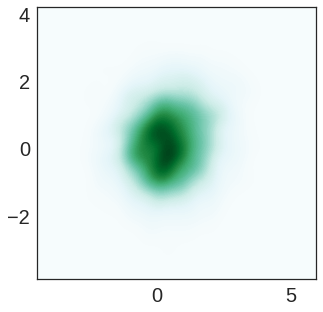}
  \caption{ALI\label{fig:ali_ring}}
\end{subfigure}
\begin{subfigure}{0.18\textwidth}
  \includegraphics[width=\linewidth]{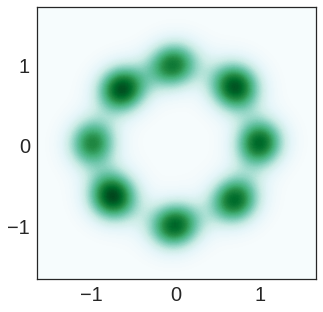}
  \caption{Unrolled\label{fig:unrolled_ring}}
\end{subfigure}
\begin{subfigure}{0.18\textwidth}%
  \includegraphics[width=\linewidth]{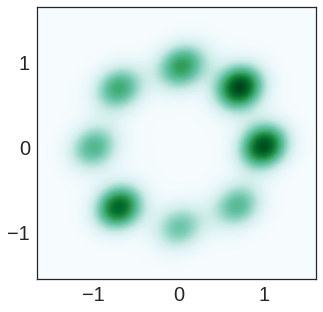}
  \caption{\acronym\label{fig:veegan_ring}}
\end{subfigure} \\


\begin{subfigure}{0.18\textwidth}
  \includegraphics[width=\linewidth]{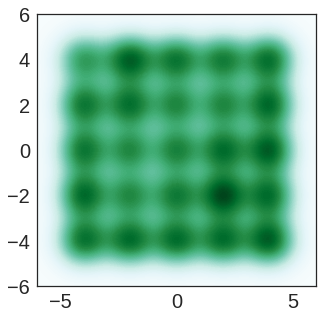}
  \caption{True Data\label{fig:true_grid2}}
\end{subfigure} 
\begin{subfigure}{0.18\textwidth}
  \includegraphics[width=\linewidth]{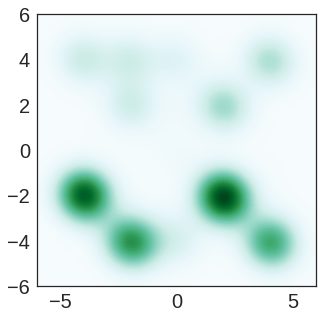}
  \caption{GAN\label{fig:gan_grid2}}
\end{subfigure}
\begin{subfigure}{0.18\textwidth}
  \includegraphics[width=\linewidth]{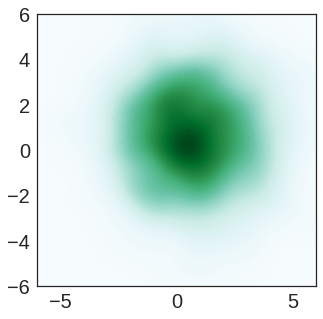}
  \caption{ALI\label{fig:ali_grid2}}
\end{subfigure}
\begin{subfigure}{0.18\textwidth}
  \includegraphics[width=\linewidth]{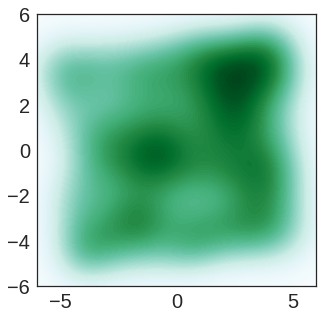}
  \caption{Unrolled\label{fig:unrolled_grid2}}
\end{subfigure}
\begin{subfigure}{0.18\textwidth}%
  \includegraphics[width=\linewidth]{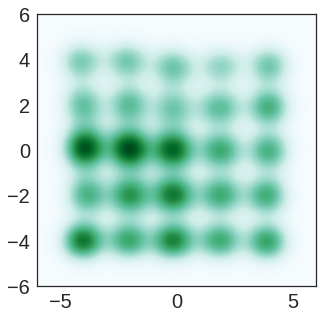}
  \caption{\acronym\label{fig:veegan_grid2}}
\end{subfigure}
\end{figure}



\begin{figure}[!htb]
\centering
\caption{Sample images from GANs trained on CIFAR-10. Best viewed magnified on screen.}\label{fig:cifar}
\begin{subfigure}{3.25in}
  \includegraphics[width=0.475\linewidth]{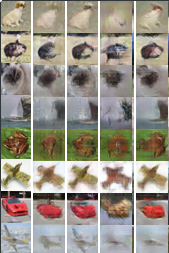}\hspace{0.5em}%
  \includegraphics[width=0.475\linewidth]{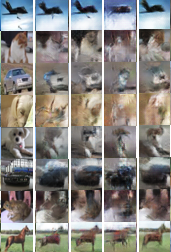}\vspace{1ex}
\caption{Generated samples nearest to real images from CIFAR-10. In each of the two panels, the first column are real images, followed by 
the nearest images from DCGAN, ALI, Unrolled GAN, and \acronym respectively. }
  \end{subfigure}\hspace{2em}%
\begin{subfigure}{1.9in}
    \includegraphics[width=\linewidth]{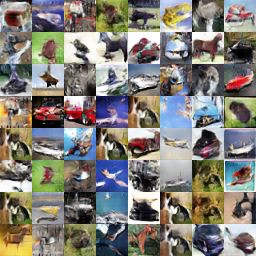}\vspace{1ex}
      \caption{Random samples from generator of \acronym trained on CIFAR-10.}\label{fig:cifar_generated}
  \end{subfigure}
\end{figure}

\section{Conclusion}
We have presented \acronym, a new training principle for GANs that 
combines a KL divergence in the joint space of representation
and data points with an autoencoder over the representation space,
motivated by a variational argument.
Experimental results on synthetic data and real images show that our
approach is much more effective than several state-of-the art GAN methods at avoiding mode collapse
while still generating good quality samples.

\section*{Acknowledgement}
We thank Martin Arjovsky, Nicolas Collignon, Luke Metz, Casper Kaae Sønderby, Lucas Theis, Soumith Chintala, Stanisław Jastrzębski, Harrison Edwards, Amos Storkey and Paulina Grnarova for their helpful comments. We would like to specially thank Ferenc Huszár for insightful discussions and feedback.

\bibliography{nips} \bibliographystyle{icml2017}

\ifdraft{\appendix
\section{Proof of Lower Bound}

This appendix completes the proof of the bound in the text that
\begin{align}
-\int p_0(z) \log p_\theta(z) 
& \leq \KL{q_\gamma(x|z) p_0(z)}{p_\theta(z|x)p(x)}
 - \Exp{\log p_0(z)}\label{eq:appendix1} 
\end{align}
where $p_0$ is the standard normal density, and $p_\theta(z) = \int p_\theta(z|x) p(x) \,dx$.
As described in the text, introducing a a variational distribution $q_\gamma (x|z)$ yields
\begin{equation}
-\int p_0(z) \log p_\theta(z) \,dz \leq -\iint p_0(z) q_\gamma(x|z) \log \frac{p_\theta(z|x)p(x)}{q_\gamma(x|z)} \,dx \,dz.\label{eq:jensen}
\end{equation}
Starting from \eqref{eq:jensen}, we obtain a new upper bound by adding a trivial KL divergence to the right hand side of the above inequality
\begin{align}
-\int p_0(z) \log p_\theta(z) \,dz
  &\leq -\iint p_0(z) q_\gamma(x|z) \log \frac{p_\theta(z|x)p(x)}{q_\gamma(x|z)} \,dx \,dz \nonumber\\
  &= \iint p_0(z) q_\gamma(x|z) \log \frac{q_\gamma(x|z)}{p_\theta(z|x)p(x)} \,dx \,dz + \int p_0(z) \log \frac{p_0(z)}{p_0(z)} \,dz \label{eq:top-ineq}
\end{align}
Now for the upper term in the KL, we have that
\begin{align*}
	\int p_0(z) \log {p_0(z)} \,dz = \int p_0(z) \log {p_0(z)} \left(\int q_\gamma(x|z) \,dx \right)\,dz 
	 = \iint p_0(z) q_\gamma(x|z) \log {p_0(z)} \,dx \,dz.
\end{align*}
Combining with \eqref{eq:top-ineq} yields
\begin{align} 
H(Z, F_\theta(X)) &\leq \iint p_0(z) q_\gamma(x|z) \log \frac{q_\gamma(x|z)}{p_\theta(z|x)p(x)} \,dx \,dz
        + \iint p_0(z) q_\gamma(x|z) \log {p_0(z)} \,dx \,dz \nonumber\\
        & \qquad\quad - \int p_0(z) \log p_0(z) \,dz \nonumber\\
        &= \iint p_0(z) q_\gamma(x|z) \log \frac{q_\gamma(x|z) p_0(z)}
        {p_\theta(z|x)p(x)} \,dx \,dz
    - \int p_0(z) \log p_0(z) \,dz \nonumber\\
    &=  \KL{q_\gamma(x|z) p_0(z)}{p_\theta(z|x)p(x)} - \int p_0(z) \log p_0(z) \,dz, \nonumber
\end{align}	
which completes the proof.

\section{Proof of Proposition 1}
\begin{proposition}
Suppose that there exist parameters $\theta^*, \gamma^*$ such that 
$\calO(\gamma^*, \theta^*)= H[p_0],$ where $H$ denotes Shannon entropy. Then $(\gamma^*, \theta^*)$ minimizes $\calO$, and we further have that
\begin{align*}
 p_{\theta^*}(z) &:= \int p_{\theta^*}(z|x) p(x) \,dx = p_0(z) \\
q_{\gamma^*}(x) &:= \int q_{\gamma^*}(x|z) p_0(z) \,dz = p(x).
\end{align*}
\end{proposition}

\begin{proof}
From information theory, we know that
$ \KL{q_\gamma(x|z) p_0(z)}{p_\theta(z|x)p(x)} \geq 0.$
Additionally, we have that $\Exp{d(z, F_\theta(x))} \geq 0,$. Moreover, by definition of $\Exp{}$ in the proposition,
\begin{align*}
-\Exp{\log p_0(z)}&= -\iint p_0(z) q_\gamma(x|z) \log p_0(z) \,dz dx = -\int p_0(z) \log p_0(z) \,dz  \int q_\gamma(x|z) \, dx \\
&= -\int p_0(z) \log p_0(z) \,dz,
\end{align*}
which is the definition of the Shannon entropy $H[p_0]$ of $p_0$.

This implies that
\begin{align*}
\calO(\gamma, \theta) &=  \KL{q_\gamma(x|z) p_0(z)}{p_\theta(z|x)p(x)}
 - \Exp{\log p_0(z)}
+ \Exp{d(z, F_\theta(x))} \\
&\geq - \Exp{\log p_0(z)} \\ &= H[p_0].
\end{align*}
This bound is attained with equality when $q_\gamma(x|z) p_0(z) = p_\theta(z|x)p(x)$,
and when $F_\theta$ inverts $G_\gamma$ on the data distribution, i.e.,
when  $F_\theta(G_\gamma(z)) = z$
for all $z$. (Note that this
statement does not require
$G$ to be invertible outside
of its range.)

Now, if $\calO(\gamma^*, \theta^*) = H[p_0],$ subtracting the entropy
from both sides implies that
$\KL{q_\gamma(x|z) p_0(z)}{p_\theta(z|x)p(x)} = 0.$
Because the optimum of the KL divergence  is unique, we then have that  $q_{\gamma^*}(x|z) p_0(z) = p_{\theta^*}(z|x)p(x).$

Integrating both sides over $x$ yields the first equality in the proposition,
and integrating over $z$ yields the second.
\end{proof}

\section{Discriminator Architecture for ALI and \acronym}
When using ALI and VEEGAN, the original DCGAN discriminator needs to be augmented in order allow it to operate on pairs of images and noise vectors.
In order to achieve this, we flatten the final convolutional layer of DCGAN's discriminator and concatenate it with the input noise vector. Afterwards, we run the concatenation through a hidden layer, and then compute $D_\omega(z,x)$ through a linear transformation.

\begin{table}[h]
\centering
\caption{ALI and \acronym Discriminator Architecture.}
\label{tab:arch}
\begin{tabular}{|c|c|c|c|}
\hline
\multicolumn{1}{|l|}{\textbf{Operation}} & \textbf{\#Output} & \multicolumn{1}{l|}{\textbf{BN?}} & \multicolumn{1}{l|}{\textbf{Activation}} \\ \hline
{$D_\omega(x)$}                            & \multicolumn{3}{l|}{}                                                                            \\ \hline
Conv                                     & 64                & False                             & Leaky ReLU                               \\ \hline
Conv                                     & 128               & True                              & Leaky ReLU                               \\ \hline
Conv                                     & 256               & True                              & Leaky ReLU                               \\ \hline
Conv                                     & 512               & True                              & Leaky ReLU                               \\ \hline
Flatten                                  & -                 & -                                 & -                                        \\ \hline
\textbf{$\sigma(D_\omega(z,x))$}                          & \multicolumn{3}{l|}{Concatenate $D_\omega(x)$ and $z$ along the first axis.}                      \\ \hline
Fully Connected                          & 512               & False                             & Leaky ReLU                               \\ \hline
Fully Connected                          & 1                 & False                             & Sigmoid                                  \\ \hline
\end{tabular}
\end{table}

\section{Inference}
While not the focus of this work, our method can also be used for inference as in the case of ALI and BiGAN models. Figure \ref{fig:ali_inf} shows an example of inference on MNIST. The top row samples are from the dataset. We extract the latent representation vector for each of the real images by running them through the trained reconstructor and then use the resulting vector in the generator to get the generated samples shown in the bottom row of the figure.
\begin{figure}
  \includegraphics[width=\linewidth]{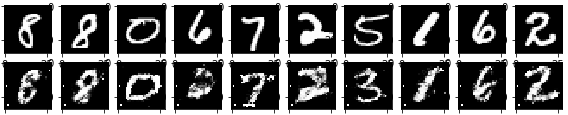}
  \caption{\acronym method can be used like ALI to perform inference. The means output from the reconstructor network for the real images in the top row are used as the latent features to samples the generated images in the bottom row. \AS{Either include reconstructions for other methods, or cut this figure entirely.
  What does this demonstrate that the other figures do not?}}
  \label{fig:ali_inf}
\end{figure}

\section{Adversarial Methods for Autoencoders}
In order to quantify contrast the effect of autoencoding of noise in \acronym with autoencoding of data in DAE methods \citep{Che2017ModeRG,larsen2015autoencoding} we train DAE version of \acronym by simply using the reconstructor network as an inference network. As mentioned before, careful tuning of the weighing parameter $\lambda$ is needed to ensure that the $\ell_2$ loss is only working as a regularizer. Therefore, we run a parameter sweep for $\lambda$. As shown in figure \ref{fig:dae_cifar} we were not able to obtain any meaningful images for any of the tested values.   
\AS{We need a sentence or two saying why we do this}

\AS{Suggest we cut Figure 7; I cannot see how it provides
evidence for any proposition. Also, "without $\calO$"
is unclear; is this just $\lambda = 0$?}

\AS{Capture in Figure 8: Are you sure the l2 loss is dominating
or are you guessing from the images? If the latter, I would cut 
this. We don't want to speculate without being very
clear that is what we are doing. And here, it's unnecessary
to speculate. Just say the images are bad.}


\begin{figure}[!htb]
\caption{CIFAR 10 samples from GANs with data Autoencoders. We did a parameter sweep over the value of $\lambda$ but were unable to generate any meaningful images for any of the values. Figure \ref{fig:ng} is generated entirely from the $\ell_2$ loss.}\label{fig:dae_cifar}
\begin{subfigure}{0.24\textwidth}
  \includegraphics[width=\linewidth]{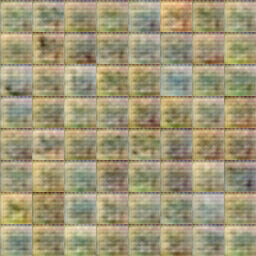}
  \caption{$\lambda=0.007$}\label{fig:7}
\end{subfigure} \hfill
\begin{subfigure}{0.24\textwidth}
  \includegraphics[width=\linewidth]{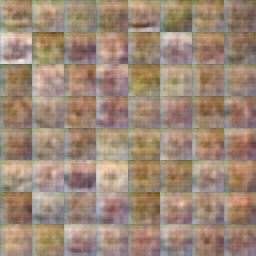}
  \caption{$\lambda=0.01$}\label{fig:1}
\end{subfigure} \hfill
\begin{subfigure}{0.24\textwidth}
  \includegraphics[width=\linewidth]{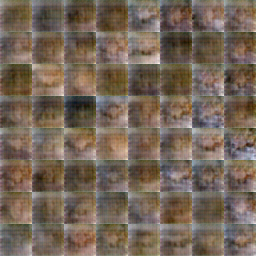}
  \caption{$\lambda=0.05$}\label{fig:5}
\end{subfigure} \hfill
\begin{subfigure}{0.24\textwidth}
  \includegraphics[width=\linewidth]{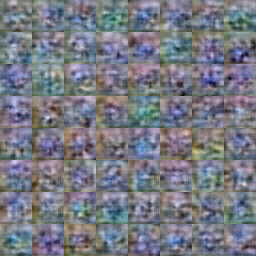}
  \caption{Only $\ell_2$}\label{fig:ng}
\end{subfigure}
\end{figure}

\section{Stacked MNIST Qualitative Results}
Qualitative results from the Stacked MNIST dataset for all the 4 methods.
\begin{figure}[!htb]
\caption{Samples from trained models for Stacked MNIST dataset.}\label{fig:sm}
\begin{subfigure}{0.18\textwidth}
  \includegraphics[width=\linewidth]{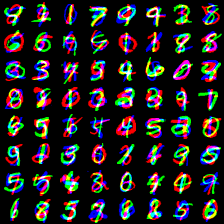}
  \caption{True Data}\label{fig:true_sm}
\end{subfigure} 
\begin{subfigure}{0.18\textwidth}
  \includegraphics[width=\linewidth]{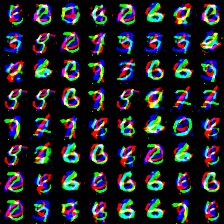}
  \caption{DCGAN}\label{fig:gan_sm}
\end{subfigure}
\begin{subfigure}{0.18\textwidth}
  \includegraphics[width=\linewidth]{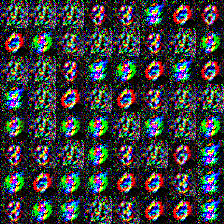}
  \caption{ALI}\label{fig:ali_sm}
\end{subfigure}
\begin{subfigure}{0.18\textwidth}
  \includegraphics[width=\linewidth]{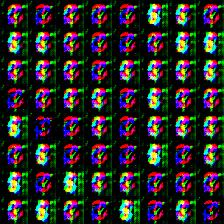}
  \caption{Unrolled}\label{fig:unrolled_sm}
\end{subfigure}
\begin{subfigure}{0.18\textwidth}%
  \includegraphics[width=\linewidth]{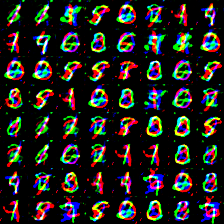}
  \caption{\acronym}\label{fig:veegan_sm}
\end{subfigure}
\end{figure}

\section{CelebA Random Sample from ALI and \acronym}

Additionally, we compared ALI and \acronym models  on the much bigger CelebA dataset \citep{liu2015faceattributes} of faces.
Our goal is to test how robust each method is when used without  extensive tuning of model architecture and hyperparameters
on a new dataset.
Therefore we use the same model architectures and hyperparameters
as we did on the CIFAR-10 data.
While ALI failed to produce any meaningful images, \acronym generates high quality images of faces. Please note that this does not mean that ALI fails on CelebA in general. Indeed, as \cite{ali2016} show, given higher capacity reconstructor and discriminator with the right  hyperparameters, it 
is possible to generate good quality images on this dataset. 
Rather, this experiment only suggests that for the simple network that we use for Stacked MNIST and CIFAR experiments, \acronym learning method was able to produce reasonable images without any further tuning or hyper parameter search.
\begin{figure}[H]
\centering
  \includegraphics[width=.6\linewidth]{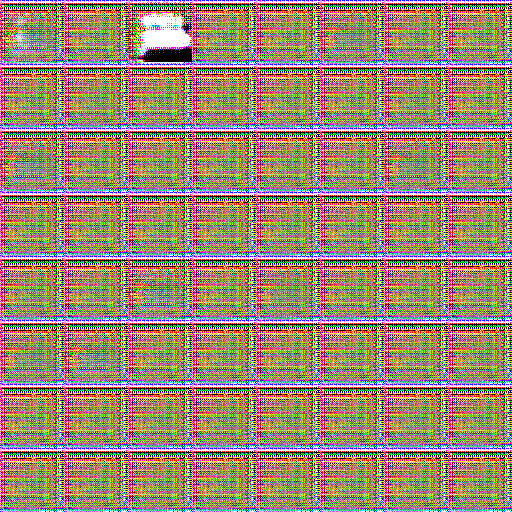}
  \caption{ALI on CelebA with simple DCGAN architecture
  and without tweaking of hyperparameters.}
  \label{fig:dcgan_cifar}
\end{figure}
\begin{figure}[H]
\centering
  \includegraphics[width=.6\linewidth]{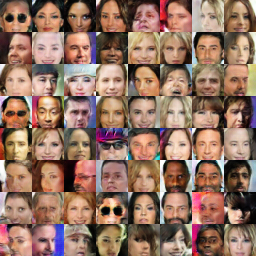}
  \caption{\acronym on CelebA with simple DCGAN architecture
  and default hyperparameters.}
  \label{fig:ali_cifar}
\end{figure}

\section{CIFAR 10 Random Sample from \acronym}

Randomly generated samples for CIFAR 10 dataset for all the 4 methods.
\begin{figure}[H]
\centering
  \includegraphics[width=.6\linewidth]{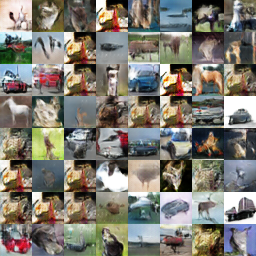}
  \caption{DCGAN on CIFAR 10 Dataset }
  \label{fig:dcgan_cifar}
\end{figure}
\begin{figure}[H]
\centering
  \includegraphics[width=.6\linewidth]{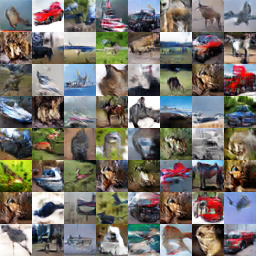}
  \caption{ALI on CIFAR 10 Dataset }
  \label{fig:ali_cifar}
\end{figure}
\begin{figure}[H]
\centering
  \includegraphics[width=.6\linewidth]{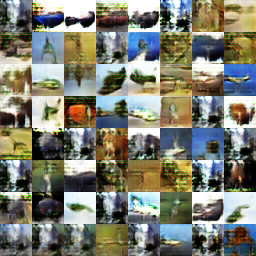}
  \caption{Unrolled GAN on CIFAR 10 Dataset }
  \label{fig:unrolled_cifar}
\end{figure}
\begin{figure}[H]
\centering
  \includegraphics[width=.6\linewidth]{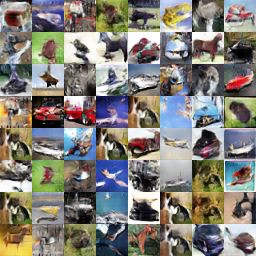}
  \caption{\acronym on CIFAR 10 Dataset }
  \label{fig:veegan_cifar}
\end{figure}

}{}

\end{document}